%% file: RobustMissions.tex
\documentclass[letterpaper, 10 pt, conference]{ieeeconf}

\IEEEoverridecommandlockouts                              
\newtheorem{theorem}{Theorem}[section]

\newtheorem{prop}[theorem]{Proposition}

\newtheorem{problem}{Problem}
\newtheorem{definition}[theorem]{Definition}
\newtheorem{rem}[theorem]{Remark}
\newtheorem{ex}[theorem]{Example}

\usepackage{graphicx} 
\usepackage{amsmath}
\usepackage{amssymb}
\usepackage{epstopdf}
\usepackage{cite}
\usepackage[noend,ruled,linesnumbered]{algorithm2e}
\SetKwComment{Comment}{$\triangleright$\ } {}
\usepackage{multirow}
\usepackage{rotating}
\usepackage{subfigure} 
\usepackage{color} 
\usepackage{mysymbol}
\usepackage[dvipsnames]{xcolor}
\usepackage[hyphens]{url}

\usepackage[breaklinks=true, colorlinks, bookmarks=true, citecolor=Black, urlcolor=Violet,linkcolor=Black]{hyperref}

\usepackage{comment}
\usepackage[colorinlistoftodos,prependcaption,textwidth=1.5cm,textsize=tiny]{todonotes}
\begin{document}
\title{\LARGE \bf Resilient Temporal Logic Planning in the Presence of Robot Failures}
\author{Samarth Kalluraya$^1$, George J. Pappas$^2$, Yiannis Kantaros$^1$
\thanks{$^{1}$Authors are with the Department of Electrical and Systems Engineering, Washington University at St. Louis, St. Louis, MO, 63130, USA. 
        {\tt\small {k.samarth,ioannisk@wustl.edu}}, $^{2}$Author is with the Department of Electrical and Systems Engineering, University of Pennsylvania, Philadelphia, PA, 19104, USA. 
        {\tt\small pappasg@seas.upenn.edu}. This work was supported by the ARL grant DCIST CRA W911NF-17-2-0181.}%
}

\maketitle 
\begin{abstract}
Several task and motion planning algorithms have been proposed recently to design paths for mobile robot teams with collaborative high-level missions specified using formal languages, such as Linear Temporal Logic (LTL). However, the designed
paths often lack reactivity to failures of robot capabilities (e.g., sensing, mobility, or manipulation) that can occur due to unanticipated events (e.g., human intervention or system malfunctioning) which in turn may compromise mission performance. 
To address this novel challenge, in this paper, we propose a new resilient mission planning algorithm for teams of heterogeneous robots with collaborative LTL missions. The robots are heterogeneous with respect to their capabilities while the mission requires applications of these skills at certain areas in the environment in a temporal/logical order. The proposed method designs paths that can adapt to unexpected failures of robot capabilities. This is accomplished by re-allocating sub-tasks to the robots based on their currently functioning skills while minimally disrupting the existing team motion plans. 
We provide experiments and theoretical guarantees demonstrating the efficiency and resiliency of the proposed algorithm. 
\end{abstract}
\IEEEpeerreviewmaketitle
   
\section{Introduction} \label{sec:Intro}
\input{files/Intro_v2}

\section{Problem Definition} \label{sec:PF}
\input{files/PF_v2}

\section{Resilient Temporal Logic Planning Against Robot Skill Failures}\label{sec:ResilientPlanning}
\input{files/solution_v2}
\section{Experimental Validation} \label{sec:Sim}
\input{files/sim_new}
\section{Conclusion} \label{sec:Concl}
This paper proposed a temporal logic mission planning algorithm for heterogeneous robot teams that is resilient against failures of robot capabilities. The proposed algorithm aims was supported theoretically and experimentally. 

\bibliographystyle{IEEEtran}
\bibliography{SK_bib.bib}

\end{document}

%% file: files/Intro_v2.tex
\textcolor{black}{
Linear Temporal Logic (LTL) has been widely used in robot motion planning to define diverse missions beyond simple reach-avoid requirements \cite{baier2008principles}, 
such as surveillance, cooperative manipulation, and delivery \cite{aksaray2015distributed,verginis2022control,asarkaya2021temporal}.
Planning algorithms with temporal logic missions have been proposed in \cite{luo2021abstraction,kantaros2020stylus,wongpiromsarn2012receding,vasile2013sampling,kloetzer2008fully,gujarathi2022mt}, focusing on robot teams with known dynamics in known environments. Recent extensions of these works have addressed unknown and dynamic environments while assuming known system dynamics \cite{guo2013revising,kantaros2020reactive,Kalluraya2023multi,lahijanian2016iterative,li2022online}. The latter assumption has been relaxed in \cite{hasanbeig2019reinforcement,cai2023learning,sun2022neurosymbolic} using machine learning methods. 
These works lack resiliency to robot failures since they assume that the robot's capabilities remain uncompromised during deployment. However, this assumption may be violated in practice due to various sources of uncertainty robots face, including inclement weather, human interventions, and component malfunctions.}

In order to improve resiliency in mission planning,
we propose a new resilient mission planning algorithm for teams of heterogeneous robots with collaborative missions expressed as LTL formulas. The robots are heterogeneous with respect to their capabilities which may include e.g., mobility, sensing, or manipulation while the LTL formula requires them to apply their capabilities at certain areas and/or objects. 
Given an LTL mission 
the proposed algorithm designs resilient plans in the sense that they adapt to robot failures that include loss of capabilities (e.g., grasping) or complete removal of the robot (e.g., due to battery draining). 
The plans are defined as sequences of robot locations and actions. This is accomplished by re-allocating sub-tasks to the robots based on their currently functioning (if any) skills. 
%
Once the tasks are re-assigned, the previously designed paths are minimally revised to adapt to the capability failures. The proposed method aims to minimally disrupt the multi-robot behavior when failures occur, by minimizing the number of re-assignments and by avoiding re-planning for the whole team (if possible). The latter is particularly important as global task re-allocation and re-planning from scratch would be computationally prohibitive and impractical to perform at runtime, especially for large robot teams. 
%
%
We provide extensive experiments with teams of heterogeneous ground and aerial robots as well as theoretical analysis demonstrating the efficiency and resiliency of the proposed algorithm against multiple unexpected failures of robot capabilities. 

\textbf{Related works:} Several task allocation methods have been proposed recently that assign either local LTL tasks \cite{schillinger2016decomposition,banks2020multi} or sub-tasks (i.e., atomic propositions) of a global collaborative LTL mission \cite{Luo2022temporal} to robots. In these works, task assignment is performed offline, i.e., before robot deployment, while robot failures during the mission are not considered. When robot failures occur, these works can be used online to globally re-allocate tasks to the robots.
However, as discussed earlier, global task re-assignment at runtime is impractical. The closest works to ours are the ones proposed in \cite{Feifei2022failure,Zhou2022Reactive,Faruq2018Simultaneous}. Particularly, \cite{Feifei2022failure} builds a product automaton modeling the multi-robot state space, the specification space, as well as possible robot failures. Using this product system,  control strategies that are reactive to failures can be extracted. Nevertheless, that work considers homogeneous robots while product-based methods lack scalability with respect to the number of robots, the size of the environment, task complexity, and number of failures. Conceptually similar approaches are proposed in \cite{Zhou2022Reactive,Faruq2018Simultaneous}. For instance, \cite{Zhou2022Reactive} considers robots with heterogeneous abilities where robots locally react to environmental and robot state changes. However, unlike our work, in case of failures of robot skills, \cite{Zhou2022Reactive} requires global task re-allocation and re-planning over a team automaton. 

\textbf{Contributions:} 
\textit{First}, we propose a resilient temporal logic mission planning algorithm for heterogeneous robot teams against robot failures. \textit{Second}, the proposed algorithm aims to minimally disrupt the multi-robot behavior in case of failures as it avoids global re-assignment/re-planning. \textit{Third}, we provide correctness, completeness, and optimality guarantees of the proposed method.
\textit{Fourth}, we provide experiments demonstrating the efficiency of our algorithm. 

%% file: files/PF_v2.tex
\subsection{Modeling of Robots and Environment}\label{sec:PFmodelRobot}
Consider a team of $N>0$ mobile robots governed by the following dynamics: $\bbp_{j}(t+1)=\bbf_j(\bbp_{j}(t),\bbu_{j}(t))$, for all $j\in \ccalR=\{1,\dots,N\}$,  where $\bbp_{j}(t)\in\mathbb{R}^n$ stands for the state (e.g., position and orientation) of robot $j$ at discrete time $t$, and $\bbu_{j}(t)\in\mathbb{R}^b$ stands for control input. 
Hereafter, we compactly denote the dynamics of all robots as: $\bbp(t+1)=\bbf(\bbp(t),\bbu(t))$, 
%
where $\bbp(t)\in \mathbb{R}^{nN}$, $\forall t\geq 0$, and $\bbu(t)\in \mathbb{R}^{bN}$. We assume that the robot state $\bbp(t)$ is known for all time instants $t\geq0$. The robots reside in a known environment $\Omega\subseteq\mathbb{R}^d$, $d\in\{2,3\}$ with obstacle-free space denoted by $\Omega_{\text{free}}\subseteq\Omega$. We assume that $\Omega_{\text{free}}$ is populated with $M>0$ regions/objects of interests, denoted by $\ell_e$, with known locations $\bbx_e\in\Omega_{\text{free}}$, $e\in\{1,\dots,M\}$.


\subsection{Heterogeneous Robot 
Abilities and Robot Failures}\label{sec:abilities_def}
We consider robots that are heterogeneous with respect to their skills. The robots have collectively $C>0$ number of abilities amongst themselves. Each ability is represented by $c\in\{1,\dots,C\}$. For instance, $c$ can represent mobility, manipulation, fire extinguishing, or various sensing skills. 
We define the set $\ccalC=[1,\dots,c,\dots,C]$ collecting all robot capabilities. 
The skills of robot $j$ are represented by a $\bbZ_j(t)=[\zeta_1^j(t),\dots,\zeta_c^j(t),\dots,\zeta_C^j(t)]$, where $\zeta_c^j(t)$ is either equal to $1$ if robot $j$ has the ability $c$ at time $t$ and $0$ otherwise. Observe that the set of skills for the robots depends on time. 
A failure of capability $c$ for robot $j$ occurs at time $t$ if $\zeta_c^j(t-1)=1$ and $\zeta_c^j(t)=0$. We model removals of robots by setting $\bbZ_j(t)$ to be a zero vector.
%
The failures can occur at unknown time instants $t$ but we assume that vectors $\bbZ_j(t)$ are known at each time $t\geq 0$, for all robots $j$. In other words, we assume that the robots are equipped with a health monitoring system that allows them to keep track of their active/inactive abilities. Also, we assume all-to-all communication. This ensures that the team is aware of any capability failure.
\textcolor{black}{For simplicity, we assume that robots cannot apply more than one skill at a time. 
}
%
%
Based on the individual robot abilities, we partition the multi-robot system into $C$ teams. We define a robot team $\ccalT_c(t)$ at time $t$ as a set collecting the robots with $\zeta_c^j(t) = 1$
%
i.e., $\ccalT_c(t)=\{j\in \ccalR~|~\zeta_c^j(t) = 1 \}$; a robot may belong to more than one team at a time. 

\begin{figure}[t]
  \centering
    \subfigure[Starting positions]{
        \label{fig:3ra}
        \includegraphics[width=0.41\linewidth]{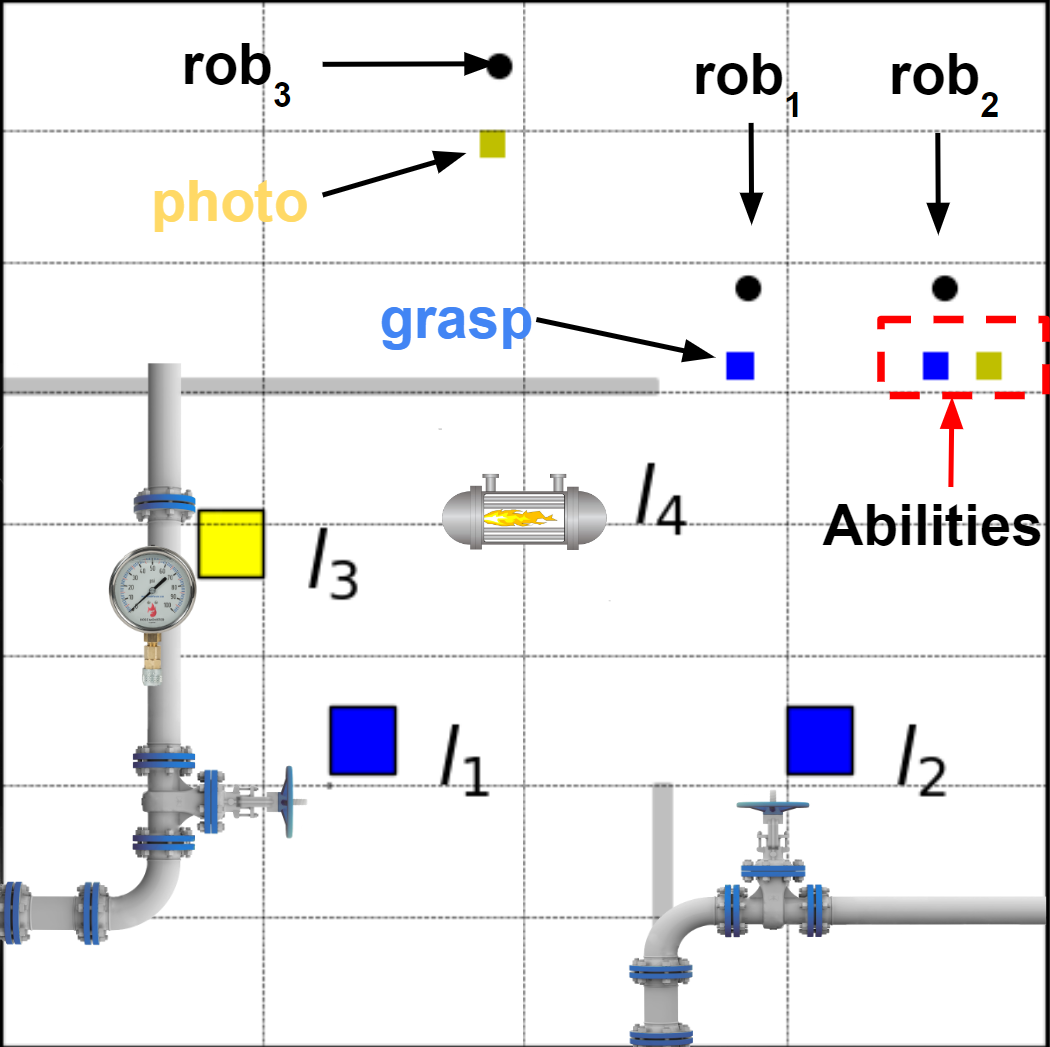}}
    \subfigure[Offline-designed paths]{      
        \label{fig:3rb}
        \includegraphics[width=0.41\linewidth]{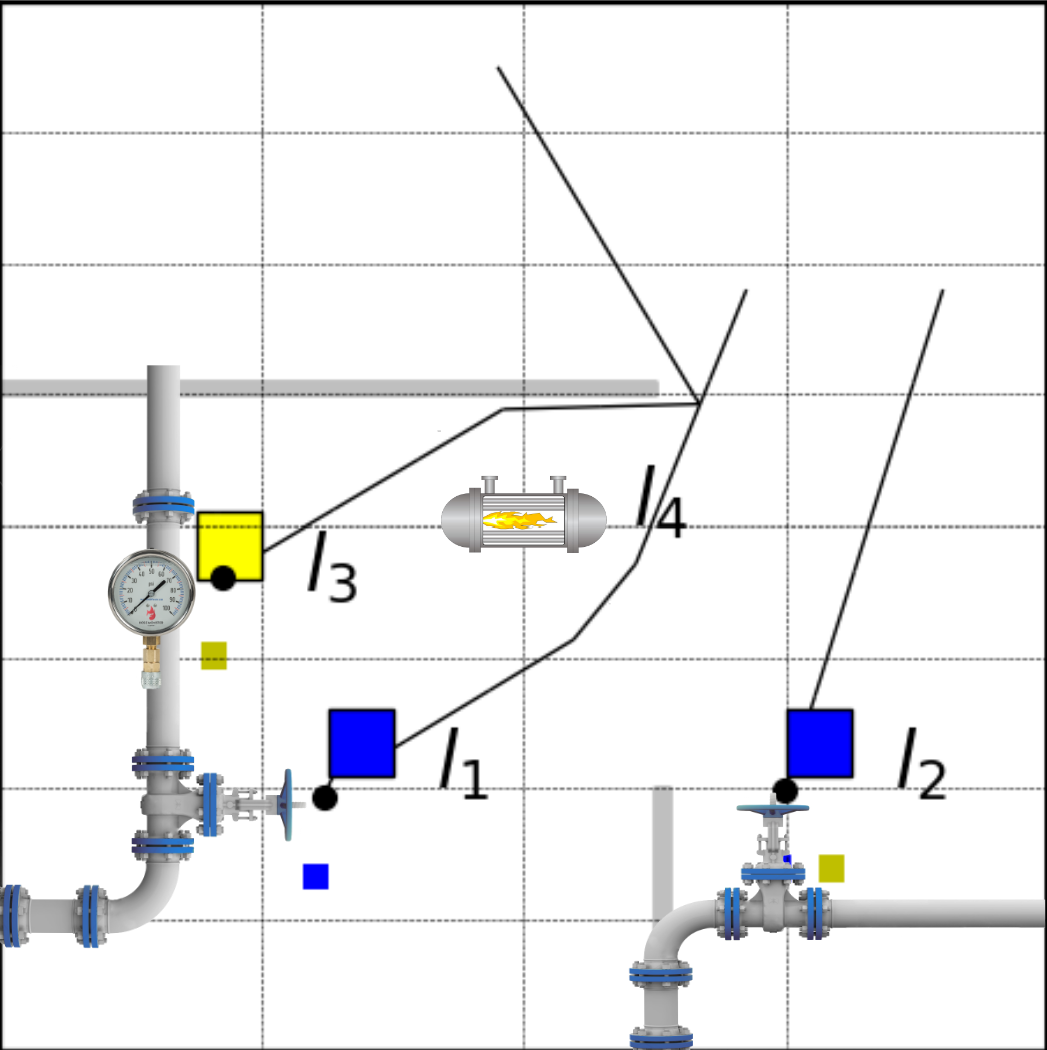}}
    \subfigure[Photo-taking skill failure]{
        \label{fig:3rc}
        \includegraphics[width=0.41\linewidth]{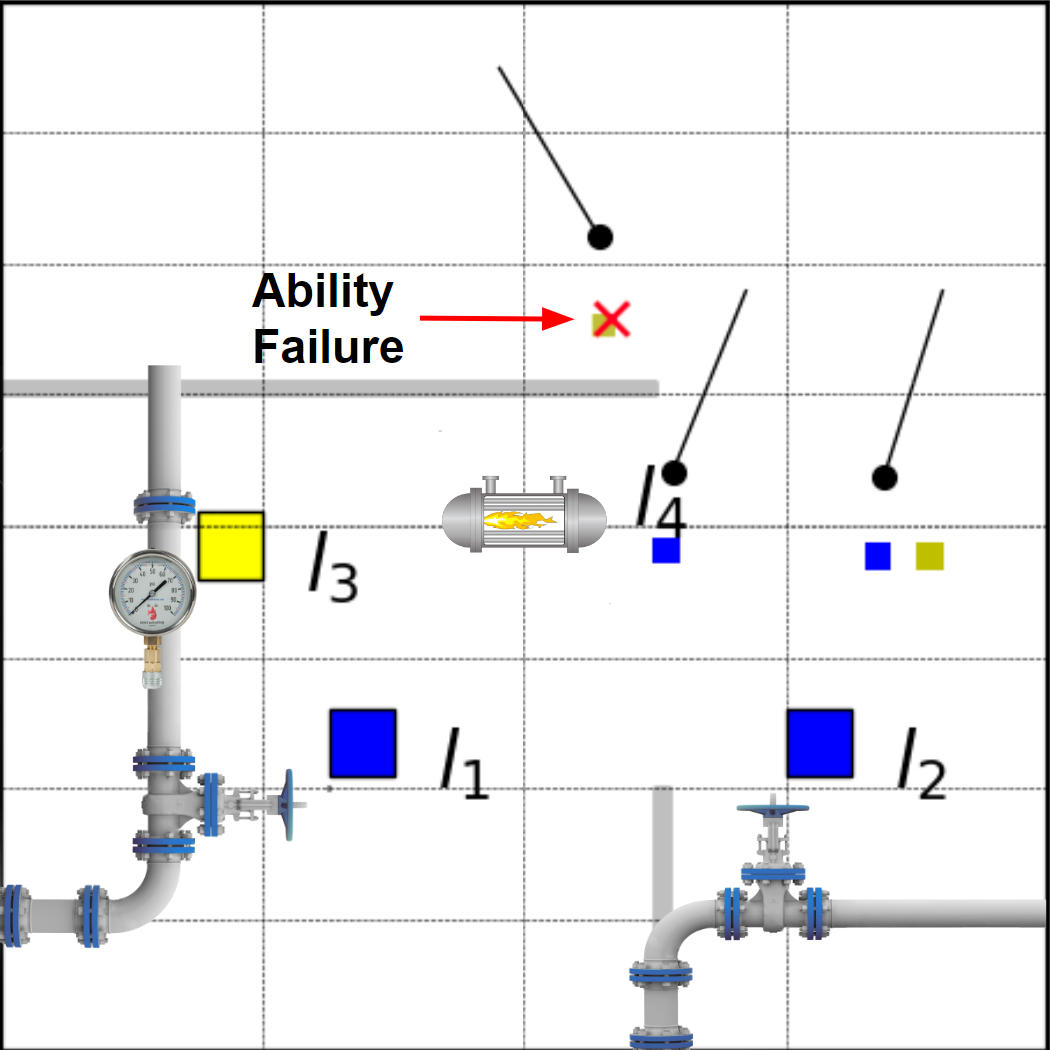}}
    \subfigure[Online Revised paths]{
        \label{fig:3rd}
        \includegraphics[width=0.41\linewidth]{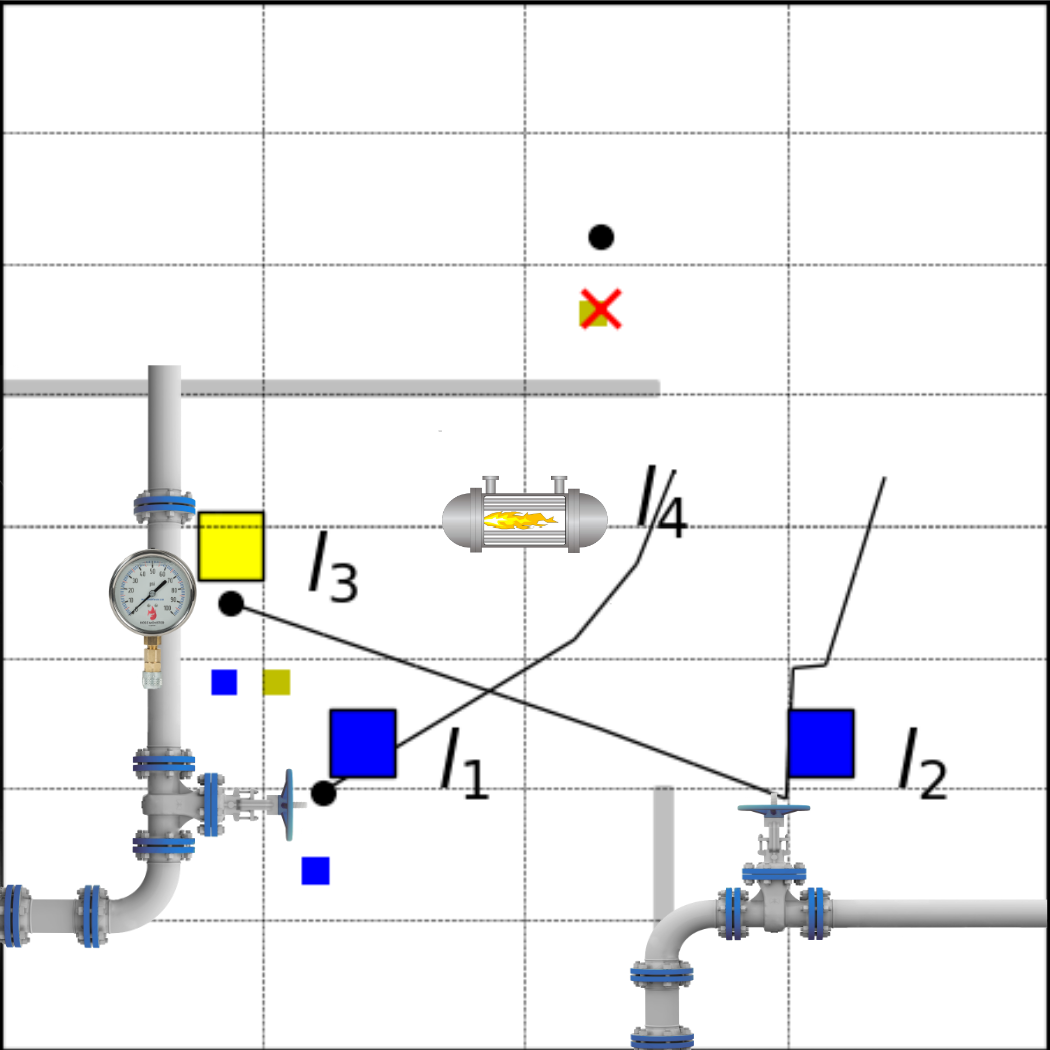}}   
 \caption{
Example \ref{ex:setup}: The small squares below each robot indicate the abilities each robot possesses; the ability to operate valves and take photos are denoted by the blue and the yellow square, respectively. The locations are represented by the larger squares, whose color indicates which ability needs to be used at that location as per $\phi$.
}
\label{fig:3robot}
\end{figure}

\subsection{Mission Specification}\label{sec:PFLTLMission}
The goal of the robots is to accomplish a high-level long-horizon collaborative mission captured by a global Linear Temporal Logic (LTL) specification $\phi$, that requires them to apply their capabilities at specific regions/objects $\ell_e$ in a temporal and logical fashion.
LTL comprises a set of atomic propositions (i.e., Boolean variables), denoted by $\mathcal{AP}$, Boolean operators, (i.e., conjunction $\wedge$, and negation $\neg$), and two temporal operators, next $\bigcirc$ and until $\mathcal{U}$. LTL formulas over a set $\mathcal{AP}$ can be constructed based on the following grammar: $\phi::=\text{true}~|~\pi~|~\phi_1\wedge\phi_2~|~\neg\phi~|~\bigcirc\phi~|~\phi_1~\mathcal{U}~\phi_2$, where $\pi\in\mathcal{AP}$. For brevity, we abstain from presenting the derivations of other Boolean and temporal operators, e.g., \textit{always} $\square$, \textit{eventually} $\lozenge$, \textit{implication} $\Rightarrow$, which can be found in \cite{baier2008principles}. 
We consider LTL tasks constructed based on the following team-based atomic predicate: 
\begin{equation}\label{eq:pip}
\pi_{\ccalT_c}(j, c, \ell_e)=
 \begin{cases}
  \text{true}, & \text{if $j\in\ccalT_c \text{ applies } c\text{ at }\ell_e$}\\
  \text{false}, & \text{otherwise}
 \end{cases}       
\end{equation}
The predicate in \eqref{eq:pip} is true when \textit{any} robot in the team $\ccalT_c$ applies the skill $c$ (e.g., `grasp') at the region/object $\ell_e$. The robot that has been assigned with this sub-task/predicate is denoted by $j$ in \eqref{eq:pip}. 
We assume an initial assignment of  robots $j$ to predicates \eqref{eq:pip} is given and that the resulting mission $\phi$ is feasible; this is a common assumption in the majority of related works. 
Based on the definition of this predicate, we also have that $\neg \pi_{\ccalT_c}(j, c, \ell_e)$ is satisfied if \textit{none} of the robots in $\ccalT_c$ applies the skill $c$ at $\ell_e$. In other words, if a predicate appears with a negation in front of it, then it concerns the whole team $\ccalT_c$ and not a specific robot $j$; i.e., the definition of robot $j$ in \eqref{eq:pip} is redundant and, therefore, replaced by the symbol $\varnothing$.
We also assume that the robots have perfect skills, i.e., once robot $j$ applies the skill $c$ at $\ell_e$, then \eqref{eq:pip} is satisfied. An example of a simple LTL mission is given in Examples \ref{ex:setup}-\ref{ex:LTL}. Given an LTL formula, a multi-robot path $\tau$ can be constructed using existing methods such as \cite{luo2021abstraction}.
The path $\tau$ is defined as an infinite sequence of states i.e., $\tau=\tau(0),\dots,\tau(t)\dots$. In $\tau$, each state $\tau(t)$ is defined as $\tau(t)=[\bbp(t),\bbs(t)]$, 
where $\bbp(t)$ is the multi-robot system state and $\bbs(t)=[s_1(t),\dots,s_N(t)]$, $s_j(t)\in\ccalC$. In other words, $s_j(t)$ determines which skill robot $j$ should apply at time $t$; if robot $j$ does not need to apply any skill at time $t$, then, with slight abuse of notation, we denote this by $s_j(t)=\varnothing$.
\subsection{Problem Statement: Resilient Temporal Logic Planning}
Consider a robot team tasked with completing a mission $\phi$. As the robots execute a designed plan $\tau$, certain robot skills may fail unexpectedly; see Ex. \ref{ex:setup}-Ex. \ref{ex:LTL}. In this case, $\tau$  may no longer be feasible compromising mission performance. Our goal is to address the following problem: 
\begin{problem}\label{prob_statement}
Consider a mission $\phi$ defined over predicates \eqref{eq:pip}, an initial assignment of these predicates to robots, and an offline generated plan $\tau$ satisfying $\phi$. When failures of robot skills occur (possibly more than one at a time), design (i) an online task re-allocation method that re-assigns sub-tasks (i.e., atomic predicates) associated with the failed skills to the robots based on their current skills; (ii) revise $\tau$ to satisfy the LTL formula arising after task re-allocation.
\end{problem}

\begin{ex}[Scenario I]\label{ex:setup}
Consider a group of $N=3$ robots with simple dynamics of the form $\bbp(t+1)=\bbp(t)+\bbu(t)$. 
The abilities among the robots are denoted by $c_1$, $c_2$, and $c_3$ 
referring to mobility, grasping valves and turning them on, and photo-taking skills using a camera, respectively. The skill-based teams are defined as $\ccalT_{c_1}(0)=\{1,2,3\}$, $\ccalT_{c_2}(0)=\{1,2\}$, and $\ccalT_{c_3}(0)=\{2,3\}$.
The robots are responsible for accomplishing a pipeline inspection task modeled by the following LTL formula: $ \phi =\Diamond(\pi_1\wedge\Diamond(\pi_2))\wedge\Diamond\pi_3 \wedge \square \neg\pi_4$
where $\pi_1=\pi_{\ccalT_2}(1, c_2, \ell_1)$, $\pi_2=\pi_{\ccalT_3}(3, c_3, \ell_3)$,
$\pi_3=\pi_{\ccalT_2}(2, c_2, \ell_2)$ 
\eqref{eq:pip}, and $\pi_4=\pi_{\ccalR}(\varnothing, c_1, \ell_4)$. 
In words, it requires robot $1$, to turn on the valve at location $\ell_1$ of a pipeline (modeled by $\pi_1$), and, subsequently, robot $3$ to take a photo  of the flow gauge there (modeled by $\pi_2$). Eventually, robot $2$ must turn on a valve at another location $\ell_2$ (modeled by $\pi_3$). At all times all robots must avoid moving close to $\ell_4$ ($c_1$) as there is a boiler present there, with high temperatures unfavorable for the robots (modeled by $\pi_4$). The initial locations and skills of the robots are shown in Fig. \ref{fig:3ra}. A feasible plan $\tau$ is shown in Figure \ref{fig:3rb}. As the robots execute these plans, robot capabilities may fail. Specifically, Fig. \ref{fig:3rc} shows a case where the camera of robot $3$ failed at $t=3$, i.e., $\zeta_{c_3}^3(2)=1$, $\zeta_{c_3}^3(3)=0$, and $\ccalT_{c_3}(3)=\{2\}$. Due to this failure, $\pi_2$, which was originally assigned to robot $3$, cannot be satisfied. As a result, it has to be re-assigned to another robot in team $\ccalT_{c_3}(3)$. The proposed algorithm assigns $\pi_2$ to robot $2$. Paths satisfying the mission after this re-assignment are shown in Fig.  \ref{fig:3rd}. This example is re-visited in Sec. \ref{sec:Sim}.  
\end{ex}




\begin{ex}[Scenario II]\label{ex:LTL}
Consider a team of $N=5$ robots with skills $c_1$, $c_2$, $c_3$, $c_4$, and $c_5$. The skills $c_1-c_3$ are defined as above. The remaining ones refer to picking up objects and temperature sensors.
The skill-based teams are $\ccalT_{c_1}(0)=\{1,2,3,4,5\}$, $\ccalT_{c_2}(0)=\{1,4\}$, $\ccalT_{c_3}(0)=\{1,2,4,5\}$, $\ccalT_{c_4}(0)=\{3\}$, and $\ccalT_{c_5}(0)=\{4,5\}$. Consider the LTL mission: $\phi = \Diamond(\pi_1\wedge\Diamond( (\pi_2 || \pi_3) \wedge \pi_5\wedge\Diamond (\pi_1)))  \wedge \square \neg\pi_4$, where $\pi_1=\pi_{\ccalT_2}(1, c_2, \ell_1)$, $\pi_2=\pi_{\ccalT_3}(2, c_3, \ell_2)$, $\pi_3=\pi_{\ccalT_4}(3, c_4, \ell_3)$, $\pi_4=\pi_{\ccalT_2}(\varnothing, c_1, \ell_2)$, and $\pi_5=\pi_{\ccalT_5}(5, c_5, \ell_5)$. If skill $c_3$ of robot $2$ and $c_4$ of robot $3$ fails at $t=1$, then a robot $i\in\ccalT_{c_3}(1)=\{1,4,5\}$ needs to take over $\pi_2$. Observe that although robot $4$ is not assigned any predicate, it cannot be assigned $\pi_2$ as satisfaction of $\pi_2$ by $4$ would result in satisfaction of $\pi_4$ which would violate $\phi$. \textcolor{black}{Robot $1$ also has the same constraint. As a result, only $5$ can take over $\pi_2$. However, note that robot $5$  needs to complete $\pi_5$ at the same time $\pi_2$ needs to be satisfied. Also observe that $\ccalT_{c_4}(1)=\varnothing$, and thus $\pi_3$ cannot be satisfied.} \textcolor{black}{Our proposed algorithm is designed to handle such challenging scenarios; see Sec. \ref{sec:ResilientPlanning}.
This example is re-visited in Sec. \ref{sec:ResilientPlanning}-\ref{sec:Sim}.
}
 %
%
\end{ex}






\begin{rem}[Independence of Sub-tasks]\label{rem:independ} 
Assume that some skills of robot $j$ fail. We assume that all atomic predicates associated with $j$ and its failed skills that appear in $\phi$ (irrespective of the locations $\ell_e$) can be re-allocated independently from each other, i.e., they do need to be assigned to the same robot $i\neq j$. For example, in Ex. \ref{ex:LTL}  
%
observe that the  atomic proposition $\pi_1$ appears twice in $\phi$. If $c_2$ fails for robot $1$, then  the `leftmost' and the `rightmost' proposition $\pi_1$ do not have to be satisfied by the same robot.

\end{rem}

%% file: files/solution_v2.tex
In this section we present an algorithm to solve Problem \ref{prob_statement}. Our solution consists of three parts. In the first part, we generate offline a path $\tau$ satisfying $\phi$ using our recently proposed sampling-based planner \cite{luo2021abstraction}; see Sections \ref{sec:nba}-\ref{sec:samplingAlg}. In the second part, 
we design an online local task re-allocation algorithm re-assigning sub-tasks to robots. The proposed algorithm is \text{local} in the sense that it aims to minimize the total number of sub-task re-allocations that have to happen due to the failed skill; see Sections \ref{sec:taskRealloc0}-\ref{sec:taskRealloc3}. Given the revised LTL formula, we propose an online re-planning method 
that locally revises the current robot paths to design new feasible ones. 
In Section \ref{sec:analysis}, we analyze correctness, completeness, and optimality of the proposed algorithm as well as its limitations. 

\subsection{From LTL Missions to Automata}\label{sec:nba}
Given an LTL mission $\phi$, we translate it, offline, into a Nondeterministic B$\ddot{\text{u}}$chi Automaton (NBA)  \cite{baier2008principles}. 


\begin{definition}[NBA]
A Nondeterministic B$\ddot{\text{u}}$chi Automaton (NBA) $B$ over $\Sigma=2^{\mathcal{AP}}$ is defined as a tuple $B=\left(\ccalQ_{B}, \ccalQ_{B}^0,\Sigma,\delta_B,\ccalQ_{B}^F\right)$, where $\ccalQ_{B}$ is the set of states, $\ccalQ_{B}^0\subseteq\ccalQ_{B}$ is a set of initial states, $\Sigma$ is an alphabet, $\delta_B:\ccalQ_B\times\Sigma\rightarrow2^{\ccalQ_B}$ is a non-deterministic transition relation, and $\ccalQ_{B}^F\in\ccalQ_{B}$ is a set of accepting/final states. 
\end{definition}


Next, we discuss the accepting condition of the NBA that is used to find plans $\tau$ that satisfy $\phi$. We define a labeling function $L:\mathbb{R}^{N}\times\mathcal{C}^{N}\rightarrow 2^{\mathcal{AP}}$ determining which atomic propositions are true given the multi-robot state $\bbp(t)$ and the applied skills $\bbs(t)$. 
An infinite run $\rho_B= q_B^1 \dots q_B^k \dots$ of $B$ over an infinite word $w = \sigma_0\sigma_1\sigma_2\dots\in\Sigma^{\omega}$, where $\sigma_k \in \Sigma$, $\forall k \in \mathbb{N}$, is an infinite sequence of NBA states $q_B^k$, $\forall k \in \mathbb{N}$, such that $q_B^{k+1}\in\delta_B(q_B^k, \sigma_k)$ and $q_B^0\in\ccalQ_B^0$. An infinite run $\rho_B$ is called 
\textit{accepting} if $\text{Inf} ( \rho_B) \bigcap \mathcal{Q}_B^F \neq \emptyset$, where $\text{Inf} (\rho_B)$ represents the set of states that appear in $\rho_B$ infinitely often. A plan $\tau=[\bbp(0),\bbs(0)],[\bbp(1),\bbs(1)],\dots,[\bbp(t),\bbs(t)],\dots$ is feasible if the word $w=\sigma_0\sigma_1\dots \sigma_t\dots$ where $\sigma_t=L([\bbp(t),\bbs(t)])$, results in at least one accepting run $\rho_B$. \textcolor{black}{Since we assume that robots cannot apply more than one skill at a time, in what follows, we assume that the NBA is pruned as in \cite{luo2021abstraction} by removing transitions that violate this assumption.}




\subsection{Sampling-based planner}\label{sec:samplingAlg}
Given an LTL formula $\phi$, we design feasible plans $\tau$ using the motion planner developed in \cite{luo2021abstraction} due to its abstraction-free and scalability benefits; any other motion planner can be employed. Specifically, \cite{luo2021abstraction} proposes a sampling-based planner that incrementally builds trees that explore both the robot motion space and the automaton state-space. The nodes of the tree are defined as $\bbq(t)=[\bbp(t), \bbs(t), q_B(t)]$. The root $\bbq(0)$ of the tree is defined based on the initial robot states $\bbp(0)$, a null vector $\bbs(0)$, and an initial NBA state $q_B(0)\in\ccalQ_B^0$. At every iteration of the algorithm, a new state $\bbq(t)$ is sampled and added to the tree if is feasible (i.e., it does not result in violation of $\phi$). 
This sampling-based approach is capable of generating plans, i.e., sequences of states $\bbq(t)$ in a prefix-suffix form. The prefix is executed first followed by the indefinite execution of the suffix. We denote this plan by $\tau_H=\tau_H^{\text{pre}}[\tau_H^{\text{suf}}]^{\omega}$, where $\omega$ stands for indefinite repetition. The prefix part $\tau_H^{\text{pre}}$ is defined as $\tau_H^{\text{pre}}=\bbq(0),\bbq(1),\dots, \bbq(T)$, for some horizon $T\geq 0$, where $q_B(H)\in\ccalQ_B^F$, and the suffix part $\tau_H^{\text{suf}}$ is defined as $\tau_H^{\text{suf}}=\bbq(T+1),\bbq(T+2),\dots, \bbq(T+K)$, for some $K\geq0$ where $\bbq(T+1)=\bbq(T+K)=\bbq(T)$. 
To design the prefix part, a goal region $\ccalR_{\text{pre}}$ is defined collecting all possible nodes $\bbq$ whose NBA state belongs to $\ccalQ_B^F$, i.e., $\ccalR_{\text{pre}}=\{\bbq=(\bbp,\bbs,q_B)~|~q_B\in\ccalQ_B^F\}$. Once a tree branch, i.e., a sequence of tree nodes, starting from the root and ending in $\ccalR_{\text{pre}}$ is constructed, a prefix part $\tau_H^{\text{pre}}$ can be derived. Similarly, given a prefix part  $\tau_H^{\text{pre}}$, its corresponding suffix path can be found by building a new tree rooted at the last state of $\tau_H^{\text{pre}}=\bbq(T+1)$. Then, the goal region is defined so that a cyclic path (i.e., the suffix part) around the root can be derived. Thus, the goal region, in this case, is defined as $\ccalR_{\text{suf}}=\{\bbq=(\bbp,\bbs,q_B)~|~q_B(T+1)\in \delta_B(L([\bbp,\bbs]),q_B)\}$.
Eliminating the NBA states from $\tau_H$ results in the desired plan $\tau=\tau^{\text{pre}}[\tau^{\text{suf}}]^{\omega}$; more details can be found in \cite{luo2021abstraction}. We denote by $\tau(t)=[\bbp(t),\bbs(t)]$ the state and action of the robots at time $t$ as per $\tau$. The same notation extends to $\tau_H$.

\subsection{Setting Up the Task Reallocation Process}\label{sec:taskRealloc0}
Assume that a feasible plan $\tau$ has been generated as in Section \ref{sec:samplingAlg}. Consider the case where, as the robots execute the plan $\tau$, a sub-set of robot capabilities fail at an a priori unknown time $t$ resulting in new vectors $\bbZ_j(t)$ for some robots $j\in\ccalR$. 
Let $\mathcal{AP}_F\subseteq\mathcal{AP}$ be a set collecting all atomic predicates $\pi_{\ccalT_c}(j,c,\ell_e)$ in $\phi$ that can never be true due to the failed skills, given the current assignment of predicates to robots.\footnote{Negated predicates of the form $\neg\pi_{\ccalT_c}(\varnothing,c,\ell_e)$ are not included in $\mathcal{AP}_F$ as such predicates refer to the whole team $\ccalT_c$ and not a specific robot. As a result, it is not meaningful to re-assign them.} With slight abuse of terminology, we refer to $\pi_{\ccalT_c}(j,c,\ell_e)\in\mathcal{AP}_F$ as a `failed' atomic predicate. The following process is repeated for all failed atomic predicates \textcolor{black}{sequentially and in any order (line \ref{rp:for1}, Alg. \ref{alg:RP})}. 

%

%

%
Our goal is to re-assign $\pi_{\ccalT_c}(j,c,\ell_e)\in\mathcal{AP}_F$ to a new robot $i\in\ccalT_c$; again, with slight abuse of terminology, we also refer to this process as `fixing/repairing' the failed predicate. Informally, the key challenge here is that possibly all robots $i\in\ccalT_c$ may be busy with other sub-tasks; \textcolor{black}{see Ex. \ref{ex:LTL}.} In this case, repairing the failed predicate may trigger a sequence of task re-assignments. The reason is that fixing the failed predicate requires another robot $i\in\ccalT_c$ to take over it while, at the same time, another robot may have to take over the sub-task that was originally assigned to robot $i$ and so on until there are no unassigned sub-tasks. Our goal is to fix the failed predicate by making the minimum possible number of task re-assignments to minimally disrupt mission planning.

Assume that failures occur at time $t$. Let $q_B^\text{cur}=q_B(t)$ be the NBA state that the robots have reached after executing the first $t$ steps of their plan $\tau$.
%
Given a failed predicate $\pi_{\ccalT_c}(j,c,\ell_e)$, we compute all NBA states that can be reached from $q_B^\text{cur}$ through a multi-hop path. This step can be implemented by treating the NBA as a directed graph and checking which states $q_B'\in\ccalQ_B$ can be reached from $q_B^\text{cur}$. We collect these states (including $q_B^\text{cur}$) in a set $\hat{\ccalQ}_B^{\text{cur}}\subseteq\ccalQ_B$.
%
Then, among all NBA transitions between states  $q_B',q_B''\in\hat{\ccalQ}_B^{\text{cur}}$, we collect in a set $\ccalE$ the ones for which $\pi_{\ccalT_c}(j,c,\ell_e)$ appears in the corresponding Boolean formula $b_{q_B',q_B''}$ (line \ref{rp:collect}, Alg. \ref{alg:RP}). 
With slight abuse of notation, let $e=(q_B',q_B'')$ denote an edge in $\ccalE$ (i.e., an NBA transition from $q_B'$ to $q_B''$). 
Given the edge $e\in\ccalE$, we define the Boolean formula $b_{q_B',q_B''}$ defined over $\mathcal{AP}$ for which it holds that if $q_B''\in\delta_B(q_B',\sigma)$, for some $\sigma\in\Sigma$, then $\sigma\models b_{q_B',q_B''}$ and vice versa. Such Boolean formulas can be constructed automatically using existing tools such as \cite{gastin2001fast}.
%

%
\subsection{Constraints for Task Reallocation}\label{sec:taskRealloc2}

The key idea in our approach is to inspect all edges $e\in\ccalE$ and re-assign the predicates appearing in the respective formulas $b_{q_B',q_B''}$ so that there is no unassigned predicate and that the resulting LTL formula remains feasible (i.e., the robots can still generate paths that yield accepting NBA runs). 
A necessary condition to preserve the feasibility of the LTL formula after task re-allocation is that $b_{q_B',q_B''}$ should be feasible (i.e., it can become `true') after task reallocation, $\forall e=(q_B',q_B'')\in\ccalE$; see also Rem. \ref{rem:limitations}. In other words,
there should exist a symbol $\sigma=L([\bbp,\bbs])\in\Sigma$ generated by the robots, that satisfies the revised formulas $b_{q_B',q_B''}$ arising after task re-assignment. 

Next, we will translate this condition into constraints that should be respected when a failed predicate is repaired in an edge $e=(q_B',q_B'')$ (where  $q_B'$ may coincide with $ q_B''$).  
To this end, we re-write the Boolean $b_{q_B',q_B''}$ in a disjunctive normal form (DNF), i.e., $b_{q_B',q_B''}=\bigvee_{d=1}^Db_{q_B',q_B''}^d$, for some $D>0$.\footnote{The failed predicate may appear in more than one sub-formula $b_{q_B',q_B''}^d$ but it will be repaired separately/independently for each sub-formula; see Remark \ref{rem:independ}.}
For each Boolean formula $b_{q_B',q_B''}^d$, we define the set $\ccalR_{q_B',q_B''}^d\subseteq\ccalR$ that collects robots that appear in $b_{q_B',q_B''}^d$; see also Ex. \ref{ex:Rg}. 
Given $b_{q_B',q_B''}^d$, we also define the set $\mathcal{AP}_i$ that collects all predicates of the form \eqref{eq:pip} that appear in $b_{q_B',q_B''}^d$ assuming that the ones that are associated with skills $c$ for which $\zeta_c^i(t)=1$ are all assigned to robot $i$. \textcolor{black}{Given $\mathcal{AP}_i$, we can construct the alphabet $\Sigma_i=2^{\mathcal{AP}_i}$.} For instance, if $b_{q_B',q_B''}^d=\pi_{\ccalT_c}(j,c,\ell_e)\wedge\pi_{\ccalT_m}(i,m,\ell_f)$, then $\mathcal{AP}_i=\{\pi_{\ccalT_c}(i,c,\ell_e), \pi_{\ccalT_m}(i,m,\ell_f) \}$ if $i\in\ccalT_m\cap\ccalT_c$ 
\textcolor{black}{ and $\Sigma_i=\{\pi_{\ccalT_c}(i,c,\ell_e),\pi_{\ccalT_m}(i,m,\ell_f), \pi_{\ccalT_c}(i,c,\ell_e)\pi_{\ccalT_m}(i,m,\ell_f),\epsilon\}$, where $\epsilon$ stands for the empty symbol.}
Then, we define following function modeling task re-allocation constraints.


\begin{definition} [Function $V_{q_B',q_B''}^d$]\label{def:constrainedSet}
\textcolor{black}{The set-valued function $V_{q_B',q_B''}^d:\ccalR\rightarrow \Sigma_i$, given as input a robot index $i\in\ccalR$, returns a set collecting all symbols $\sigma_i\in\Sigma_i$ that if robot $i\in\ccalR$ generates, then $b_{q_B',q_B''}^d$ will be `false' regardless of the values of the other predicates}. 
We define $V_{q_B',q_B''}^d(i)=\emptyset$ for all robots $i\in\ccalR\setminus\ccalR_{q_B',q_B''}^d$. 
\end{definition}

The sets $V^d_{q_B',q_B''}(i)$ capture \textit{constraints} on what tasks/predicates a robot $i$ can take over. Specifically, when re-assigning predicates appearing in an edge $e=(q_B',q_B'')$ (where $q_B'$ may coincide with $q_B''$)
robot $i\in\ccalR_{q_B',q_B''}^d$ cannot take over predicates that appear in \textcolor{black}{$V^d_{q_B',q_B'';}(i)$} 
as this would falsify $b_{q_B',q_B''}^d$. 
Next, we define the function $g_{q_B',q_B''}^d$ that will be used later to determine which robots are currently busy with other sub-tasks.

\begin{definition} [Function $g_{q_B',q_B''}^d$]
The function $g_{q_B',q_B''}^d:\ccalR\rightarrow \mathcal{AP}$, given as an input a robot index $i\in\ccalR$, returns a set collecting the atomic predicates that are assigned to robot $i$ in $b_{q_B',q_B''}^d$ excluding the negated ones and the failed predicate.\footnote{There is at most one predicate assigned to a robot $i\in\ccalR_{q_B',q_B''}^d$ as all NBA transitions requiring a robot to satisfy more than one predicate at a time are pruned; see Section \ref{sec:nba}.} We define $g(i)=\varnothing$, for all robots $i\not\in\ccalR_{q_B',q_B''}^d$ and for all robots $i\in\ccalR_{q_B',q_B''}^d$ appearing in negated predicates or in the failed predicate. 
\end{definition}

\begin{ex}[Function $g_{q_B',q_B''}^d$, $V_{q_B',q_B''}^d$ and sets $\ccalR_{q_B',q_B''}^d$]\label{ex:Rg}
Consider the LTL formula given in Example \ref{ex:LTL}. We focus on an NBA transition with: $b_{q_B',q_B''}^{d}=\pi_2\wedge\pi_5\wedge\neg\pi_4$, where recall  
$\pi_2=\pi_{\ccalT_3}(2, c_3, \ell_2)$, $\pi_5=\pi_{\ccalT_5}(5, c_5, \ell_5)$, and $\pi_4=\pi_{\ccalT_2}(\varnothing, c_1, \ell_2)$.
\textcolor{black}{Then, $\ccalR_{q_B',q_B''}^d=\{2,5\}\cup\ccalT_2$.} Also, $g_{q_B',q_B''}^d(i)=\varnothing$, for all robots $i\in\ccalR_{q_B',q_B''}^d\setminus\{2,5\}$ and $g(2)=\pi_2$, $g(5)=\pi_5$. We also have $V_{q_B',q_B''}^d(i)=\emptyset$ for all $i\notin\ccalT_2$ and $V_{q_B',q_B''}^d(i)=\{\pi_{\ccalT_2}(i,c_1,\ell_2), \pi_{\ccalT_3}(i,c_3,\ell_2), \pi_{\ccalT_2}(i,c_1,\ell_2)\pi_{\ccalT_3}(i,c_3,\ell_2)\}$ for all $i\in\ccalT_2\cap\ccalT_3$. Notice that $\pi_{\ccalT_2}(i,c_1,\ell_2)$ and $\pi_{\ccalT_2}(i,c_1,\ell_2)\pi_{\ccalT_3}(i,c_3,\ell_2)$ are included because of $\neg\pi_{\ccalT_2}(\varnothing, c_1, \ell_2)$ in $b_{q_B',q_B''}^{d}$. Also, $\pi_{\ccalT_3}(i,c_3,\ell_2)$ is included because if robot \textcolor{black}{$1$ or} $4$ satisfies it, then \textcolor{black}{they} will be close to location $\ell_2$; therefore, $\pi_{\ccalT_2}(i,c_1,\ell_2)$ will also be satisfied resulting in violation of  $b_{q_B',q_B''}^{d}$.
\end{ex}

\begin{algorithm}[t]
\caption{Local Task Re-allocation}
\LinesNumbered
\label{alg:RP}
\KwIn{ (i) NBA $\ccalB$, (ii) Current NBA state $q_B^{\text{cur}}$; (iii) Set of failed predicates $\mathcal{AP}_F$}
\KwOut{Revised NBA}

\For{every $\pi\in\mathcal{AP}_F$}{\label{rp:for1}
Define the ordered set of edges $\ccalE$\;\label{rp:collect}
\For{every $e=(q_B',q_B'')\in\mathcal{E}$}{\label{rp:for2}
    Rewrite: $b_{q_B',q_B''}=\bigvee_{d=1}^D b_{q_B',q_B''}^d$\;\label{rp:DNF} 
    \For{$d=1,\dots,D$}{\label{rp:for3}
    Define $\ccalG$ and functions $V^{d}_{q_B',q_B''}, g^{d}_{q_B',q_B''}$\;\label{rp:graph}
    Apply Alg. \ref{alg:bfs} to compute a sequence of re-assignments $p=p(0),\dots,p(P)$\;\label{rp:applyBFS}
    \If{$\exists p$}{
    Re-assign atomic predicates as per $p$\;\label{rp:reassign}
    }
    \Else{
    Assign $\pi=\text{False}$\;\label{rp:no_sol}
    }
    Revise $b_{q_B',q_B''}^d$\;
    }
    }
    }
\end{algorithm}

\subsection{Local Task Reallocation Algorithm}\label{sec:taskRealloc3}

In this section, we present the proposed task re-allocation algorithm. Once the set $\ccalE$ is constructed, we repeat the following steps for each edge $e\in\ccalE$ \textcolor{black}{in parallel (line \ref{rp:for2}, Alg. \ref{alg:RP})}. 
Given $e$, we express its corresponding Boolean formula in a DNF form $b_{q_B',q_B''}=\bigvee_{d=1}^Db_{q_B',q_B''}^d$ (whether $q_B'=q_B''$ or not) (line \ref{rp:DNF}, Alg. \ref{alg:RP}). Then we repeat the following steps \textcolor{black}{in parallel} for each sub-formula $b_{q_B',q_B''}^d$ (line \ref{rp:for3}, Alg. \ref{alg:RP}).

First, we define the following directed graph capturing all possible reassignments \textcolor{black}{in $b_{q_B',q_B''}^d$} \textcolor{black}{(even ones that violate constraints discussed in Section \ref{sec:taskRealloc2})}. Then, we will search over it to find  re-assignments of the sub-tasks/predicates appearing in $b_{q_B',q_B''}^d$ \textcolor{black}{so that there are no unassigned predicates and all assignments respect the constraints captured by $V_{q_B',q_B'}^d$.}
We denote this graph by $\ccalG=\{\ccalV_{\ccalG},\ccalE_{\ccalG},w_{\ccalG}\}$, where $\ccalV_{\ccalG},\ccalE_{\ccalG},w_{\ccalG}$ denote the set of nodes, edges, and a cost function $w_{\ccalG}:\ccalE_{\ccalG}\rightarrow\mathbb{R}$. \textcolor{black}{The set of nodes is defined as $\ccalV_{\ccalG}=\ccalR$}.
%
\textcolor{black}{An edge from a node $a$ to $a'\neq a$ exists if $a'\in\ccalT_{c}$, where $c$ is the skill required to satisfy the predicate $g_{q_B',q_B''}^d(a)$}. 
If $g_{q_B',q_B''}^d(a)=\varnothing$, then there no outgoing edges from $a$. The physical meaning of an edge is that robot $a'$ can take over the predicate (if any) assigned to $a$ in $b_{q_B',q_B''}^d$. 
The cost function $w_{\ccalG}$ is defined so that each edge has a cost equal to $1$. 
We emphasize that we do not explicitly construct this graph; instead, the task re-allocation algorithm only requires knowledge of the nodes $a'$ \textcolor{black}{$\in\ccalT_c$} that can be reached in one hop from any node $a$ (line \ref{rp:graph}, Alg. \ref{alg:RP}). 
Then, we apply a constrained Breadth First Search (BFS) algorithm over $\ccalG$ to re-assign sub-tasks (i.e., predicates appearing in $b_{q_B',q_B''}^d$) to robots (line \ref{rp:applyBFS}, Alg. \ref{alg:RP}). 
Our goal is to find a path in $\ccalG$ from the robot associated with failed predicate, denoted by $a_{\text{root}}$, to any node $a'$ satisfying $g_{q_B',q_B''}^d(a')=\varnothing$ (i.e., $a'$ is not assigned any task in $b_{q_B',q_B''}^d$) while respecting the constraints discussed in Section \ref{sec:taskRealloc2}. \textcolor{black}{We define the set $\ccalA=\{a\in\ccalR~|~g_{q_B',q_B''}^d(a)=\varnothing\}$.} 
%
%
\textcolor{black}{Let $p=p(0),p(1),\dots,p(P)$ denote such path over $\ccalG$, where $p(0)=a_{\text{root}}$, $p(P)\in\ccalA$ and $p(k)\notin\ccalA$, for all $k\in\{2,\dots,P-1\}$. 
Such a path will dictate the re-assignment of tasks required to fix the failed predicate. 
Specifically, the robot $p(k+1)$ takes over the sub-task of robot $p(k)$ in $b_{q_B',q_B''}^d$ (i.e., the atomic predicate $g_{q_B',q_B''}^d(p(k))$). Note that this means that the robot $p(k+1)$ gives up on its current sub-tasks (which will be taken over by the robot $p(k+2)$).} \textcolor{black}{Also, since this path should respect the constraints discussed in Section \ref{sec:taskRealloc2}, it must also hold that $g_{q_B',q_B''}^d(p(k))\notin V_{q_B',q_B''}^d(p(k+1))$ for all $k\in\{2,\dots,P-1\}$.}


\begin{algorithm}[t]
\caption{Breadth First Search}
\LinesNumbered
\label{alg:bfs}
\KwIn{ (i) Failed predicate $\pi_{\ccalT_c}(j, c,\ell_e$), 
(ii) $V^{d}_{q_B',q_B''}$, (iii) $g^{d}_{q_B',q_B''}$, (iv) \textcolor{black}{Teams $\ccalT_c(t), ~\forall c\in\ccalC$}}
\KwOut{Path $p$}
$a_{\text{root}}$ = $j$\;\label{bfs:root}
$\mathcal{Q}$ = [$a_{\text{root}}$]\;\label{bfs:queue}
$\texttt{Flag}_{\text{root}}$ = True\;\label{bfs:falgSetTrue}
\While{$\sim$empty($\mathcal{Q}$)}{
    $a\leftarrow$ POP($\mathcal{Q}$)\;\label{bfs:pop}
    \If{$a\in\ccalA$ \& $\sim \texttt{Flag}_{\text{root}}$}{\label{bfs:terminate}
        Using $\texttt{Parent}$ function return path $p$ \;\label{bfs:return}
    }
    $\texttt{Flag}_{\text{root}}$=False\;\label{bfs:falgSetFalse}
    \For{$a'$ adjacent to $a$ in $\ccalG$}{
        \If{$g_{q_B',q_B''}^d(a)\notin V^{d}_{q_B',q_B''}(a')$ \& $a'$ not explored}
        {\label{bfs:conditions}
                Label $a'$ as explored\;
                $\texttt{Parent}(a') = a$\;\label{bfs:addparent}
                Append $a'$ to $\mathcal{Q}$;\label{bfs:appendToQ
            }
            
        }
        
    }
}
return ($\emptyset$) \label{bfs:failreturn}
\end{algorithm}
We apply a BFS algorithm to find the shortest path $p$ summarized in Alg. \ref{alg:bfs}; see Fig. \ref{fig:bfs}. 
Notice that there are two key differences with the standard BFS algorithm. 
The first one concerns when nodes are added to the queue data structure $\mathcal{Q}$ of the BFS algorithm. Specifically, when a node $a$ is popped/removed from $\mathcal{Q}$, then each adjacent node $a'$ is added  to $\mathcal{Q}$ if (1) it has not been explored yet (as in standard BFS); 
and (2) $g_{q_B',q_B''}^d(a)\notin V^{d}_{q_B',q_B''}(a')$ (line \ref{bfs:conditions}, Alg. \ref{alg:bfs});
%
The first constraint prevents cases where a single robot will have to replace two robots as this may result in Boolean formulas $b_{q_B',q_B''}^d$ that is satisfied if a robot applies more than one skill simultaneously which is unfeasible (see Section \ref{sec:PF}). The second constraint ensures satisfaction of the constraints discussed in Section \ref{sec:taskRealloc2}. The second difference from standard BFS is that the root node is not initially labeled as `explored'. This allows the robot with the failed skill (root) to take over other sub-tasks using its remaining active skills (if any).  
Finally, we note that the graph-search process is terminated when the first feasible path from $a_{\text{root}}$ to any node in $\ccalA$ is found (line \ref{bfs:terminate}-\ref{bfs:return}, Alg. \ref{alg:bfs}). The reason is that this path corresponds to the minimum possible re-assignments that can occur to fix the failed predicate in $b_{q_B',q_B''}^d$; see Section \ref{sec:analysis}.

Once all failed predicates are fixed, the associated formulas $b_{q_B',q_B''}$ are accordingly revised, yielding a new NBA (line \ref{rp:reassign}, Alg. \ref{alg:RP}). If a failed predicate cannot be fixed (i.e., Alg. \ref{alg:bfs} cannot find a feasible path), then this failed predicate is replaced by the logical `false' in all related formulas $b_{q_B',q_B''}$ as there is no robot that can take over it (line \ref{rp:no_sol}, Alg. \ref{alg:RP}); note that this does not necessarily imply infeasibility of the LTL formula. 
This revised NBA is an input to an online planner that designs new paths; see Sec. \ref{sec:onlineReplan}.

\subsection{Online replanning}\label{sec:onlineReplan}
Assume that the multi-robot system is at state $\tau_H(t)=[\bbp(t), \bbs(t), q_B(t)]$ when capability failures occurred and that Alg. \ref{alg:RP} has re-assigned tasks to robots. A straightforward solution to revise the robot paths is to apply the sampling-based planner, discussed in Section \ref{sec:samplingAlg}, to build a new prefix-suffix plan from scratch.
Nevertheless, re-planning from scratch for \textit{all} robots may be unnecessary given the `local' task re-allocations made by  Alg. \ref{alg:RP} while it may be impractical for large robot teams. Inspired by \cite{guo2013revising}, we propose an alternative re-planning approach that aims to locally revise the multi-robot plans; see also Fig. \ref{fig:revise}.

First, we define the path $\hat{\tau}_H=\tau_H^{\text{pre}}\tau_H^{\text{suf}}=[\bbq(1),\dots,\bbq(T)],[\bbq(T+1),\dots, \bbq(T+K)]$ that concatenate the prefix and the suffix part of the current plan $\tau_H$ (without repeating  $\tau_H^{\text{suf}}$); see also Section \ref{sec:samplingAlg}.
Second, we compute the state $\hat{\tau}_H(k)$ for which it holds $\hat{\tau}_H(k)=\tau_H(t)$. \footnote{It is possible that $t>T+K$ since the state $\tau_H(t)$ may belong to the suffix part and $\tau_H$ contains an infinite repetition of the suffix part. In $\hat{\tau}_H(k)$, $k$ points to the $k$-th entry in $\hat{\tau}_H$ where   $\hat{\tau}_H(k)=\tau_H(t)$. } Third, we define the ordered set $\ccalT_{\text{pre}}$ collecting all states $\hat{\tau}_H(k')=[\bbp(k'), \bbs(k'), q_B(k')]$ in $\hat{\tau}_H$ that satisfy the following requirements: i) $k' \geq k$; ii) ($q_B(k')\neq q_B(k'+1)$ and $(q_B(k'),q_B(k'+1))\in\ccalE$) or ($q_B(k')\neq q_B(k'-1)$ and $(q_B(k'-1),q_B(k'))\in\ccalE$); and iii) $k'\leq T$. We denote by $\ccalT_{\text{pre}}(e)$ the $e$-th entry in $\ccalT_{\text{pre}}$. 
We similarly define $\ccalT_{\text{suf}}$ where the third requirement is replaced by $T+1\leq k'\leq T+K$. The entries in these sets appear in increasing order of the indices $k'$ (i.e., $\hat{\tau}_H(k)$ always appear first).  

\begin{figure}[t]
  \centering
    \includegraphics[width=1\linewidth]{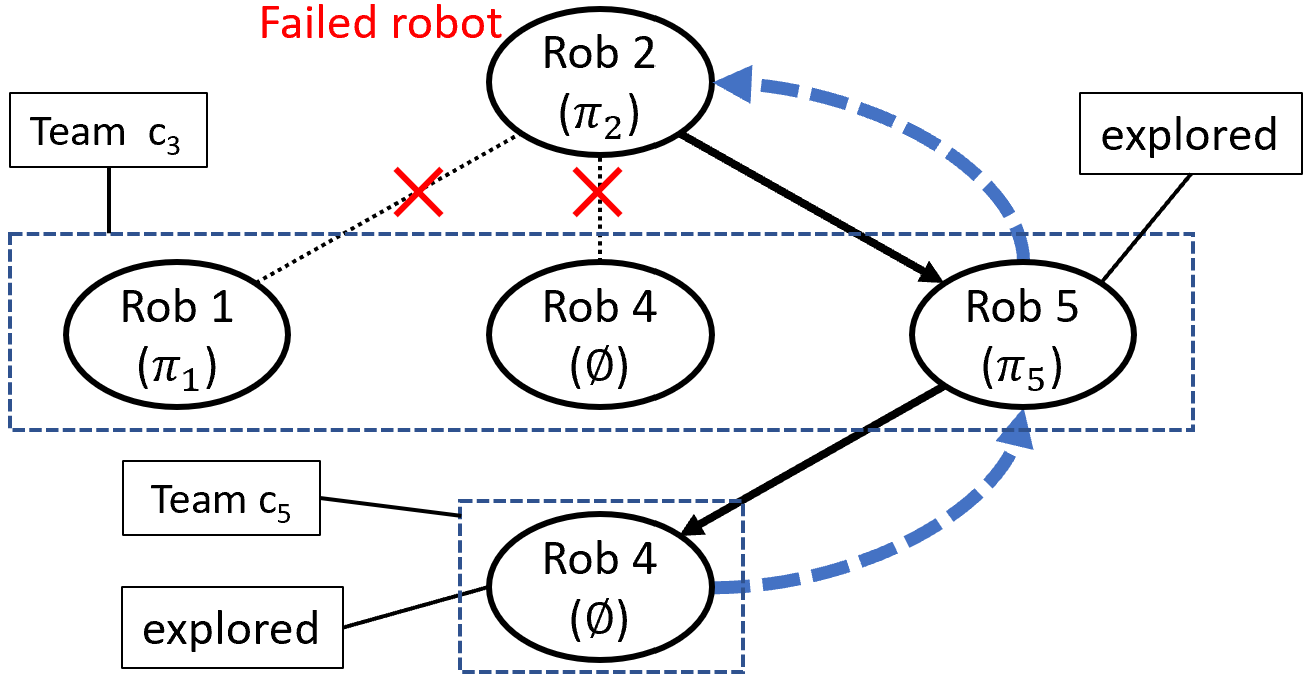}
  \caption{Consider in Example \ref{ex:LTL} the case where skill $c_3$ of robot $2$ fails, i.e., the failed predicate is $\pi_2$. We present the BFS tree (Alg. \ref{alg:bfs}) built to fix $\pi_2$ for the NBA transition enabled by  $b_{q_B',q_B''}^{d}=\pi_1\wedge\pi_2\wedge\pi_5\wedge\neg\pi_4$. The set $\ccalA$ is defined as $\ccalA=\{4\}$ and the root of the tree is robot $2$. Robots $1$, $4$, $5$ are adjacent to robot $2$ in $\ccalG$. Robots $1$ and $4$ are not connected to robot $2$ because they do not satisfy $g_{q_B',q_B''}^d(2)\notin V^{d}_{q_B',q_B''}(a')=\{\pi_2,\pi_4,\pi_2\pi_4\}$, for $a'\in\{1,4\}$. 
  Robot $5$ is connected to robot $2$ and subsequently, robot $4$ is connected to robot $5$. The blue dashed arrows show the re-assignment process along the computed path $p$, i.e., robot $4$ will take over $\pi_5$ and robot $5$ will take over the failed predicate. 
  }
  \label{fig:bfs}
\end{figure}

Assume that the failures occurred while the robots were executing the prefix part in $\tau_H$. 
%
To locally revise the prefix part, we build a tree, using the sampling-based planner discussed in Section \ref{sec:samplingAlg}, rooted at $\ccalT_{\text{pre}}(r)$, with $r=1$ and the goal region is defined as  
$\ccalR_{\text{pre}}^{\text{rev}}=\{\bbq(T)\}\cup \{\ccalT_{\text{pre}}(s)\}_{s=r+1}^{|\ccalT_{\text{pre}}|}\cup \ccalR_{\text{pre}}$, i.e., the goal is to reach either any subsequent state in $\ccalT_{\text{pre}}$ along  the current path $\hat{\tau}_H$, or reach the final state $\bbq(T)$ associated with the current path $\hat{\tau}_H$, or reach any final state in the set $\ccalR_{\text{pre}}$ defined in Section \ref{sec:samplingAlg}. We terminate the sampling-based planner as soon as it returns the first feasible path. If that path connects the root to either $\ccalR_{\text{pre}}$ or $\hat{\tau}_H(T)$ then this path corresponds to the new/revised prefix part; 
construction of all other trees terminates. Otherwise, if it finds a path towards a state $\ccalT_{\text{pre}}(s)$, $s>r$, then we build a new tree, exactly as before, rooted at $\ccalT_{\text{pre}}(r)$ where $r$ is re-defined to be $r=s$. Then, the above process is repeated until a path towards either  $\ccalR_{\text{pre}}$ or $\hat{\tau}_H(T)$ is found. The revised prefix path is constructed by replacing the affected parts in $\tau_{H}^{\text{pre}}$, as captured by $\ccalT_{\text{pre}}$, with the revised ones; see also Fig. \ref{fig:revise}. 
If no prefix path can be found, global re-planning is triggered where a new-prefix suffix plan is designed from scratch as discussed in Section \ref{sec:samplingAlg}; the only difference is that the root for the prefix tree is based on the current robot position $\bbp(t)$ and current NBA state $q_B(t)$.
%
Once the prefix part is designed, construction of the suffix part follows. If the revised prefix part ends in the state $\hat{\tau}_H(T)$, then the suffix part is constructed by following a similar logic as above. Otherwise, if it ends in a state $\ccalR_{\text{pre}}\setminus\hat{\tau}_H(T)$, the suffix part needs to be constructed from scratch by building a new tree rooted at the end state of the prefix path. 
If the failures occur when the robots are executing the suffix part in $\tau_H^{\text{suf}}$, then a similar approach is applied as before aiming to locally revise the suffix part i.e., the loop around $\hat{\tau}_H(T)$. If such a cyclic path around $\hat{\tau}_H(T)$ cannot be found, then a new prefix-suffix path is constructed from scratch.

\begin{figure}[t]
  \centering
      \subfigure[Current Path]{
    \label{fig:curPath}
  \includegraphics[width=1\linewidth]{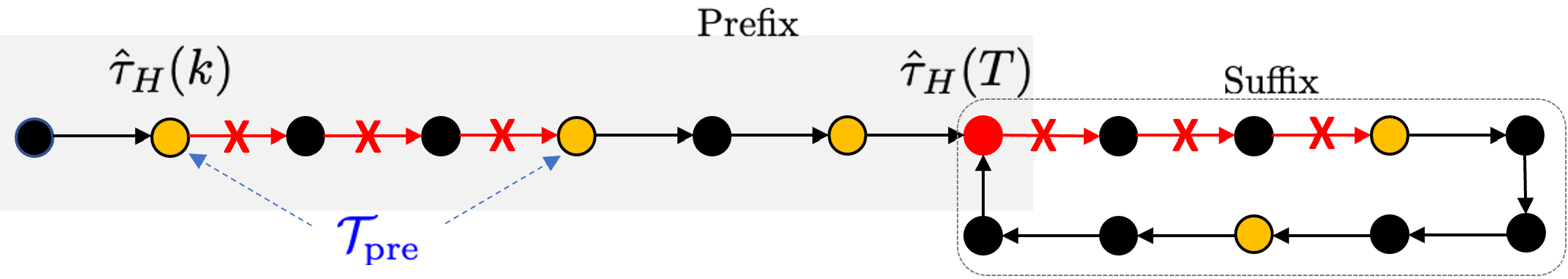}}
  \subfigure[Revised Path]{
    \label{fig:revPath}
  \includegraphics[width=1\linewidth]{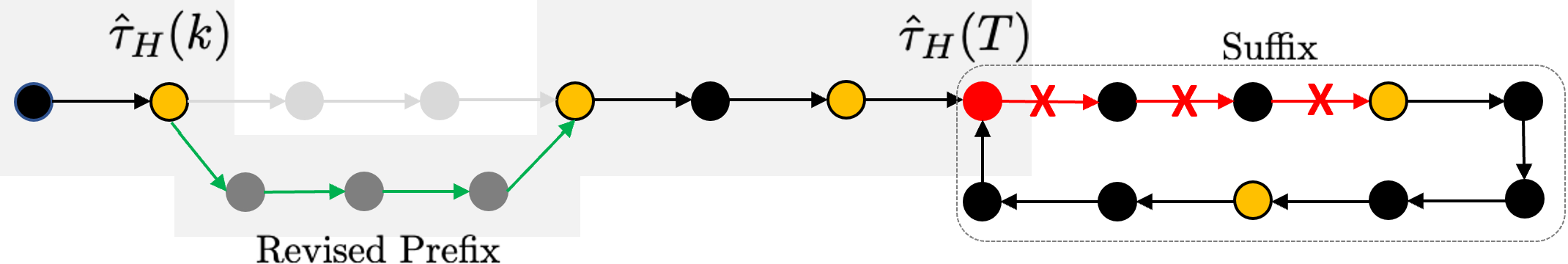}}
  \caption{Example of online revision of the prefix path. The disks capture states in $\hat{\tau}_H$.
  Yellow states $\hat{\tau}_H(k')$ model states for which it holds $q_B(k')\neq q_B(k'-1)$. The red state corresponds to a final state around which a suffix part is designed. The part of $\hat{\tau}_H$ connecting  NBA states $q_B$ and $q_B'$, where $e=(q_B,q_B')\in\ccalE$ (see Alg. \ref{alg:RP}) is marked with a red color and a red `X' denoting that it requires revision. In this example, a path was found when a tree was built rooted at $\ccalT_{\text{pre}}(1)=\hat{\tau}_H(k)$. The computed path connects $\hat{\tau}_H(k)$ to $\ccalT_{\text{pre}}(2)$. The revised prefix part is constructed by replacing the red marked edges in $\hat{\tau}_H(k)$ with the previously computed path. 
  }
  \label{fig:revise}
\end{figure}

\subsection{Algorithm Analysis}\label{sec:analysis}
%
%
In this section, we discuss correctness, completeness and optimality of the proposed algorithm. 

\begin{prop}[Completeness/Optimality of Alg. \ref{alg:bfs}]\label{bfsCompl}
Consider a failed predicate $\pi\in\mathcal{AP}_F$ and Boolean formula $b_{q_B',q_B''}^d$ that contains $\pi$. 
Alg. \ref{alg:bfs} is complete, i.e., if there exists a re-allocation of all predicates that appear in $b_{q_B',q_B''}^d$ (i.e., a path $p$) satisfying the constraints discussed in Section \ref{sec:taskRealloc2}, then Alg. \ref{alg:bfs} will find it. Algorithm \ref{alg:bfs} is also optimal,
i.e., if there exist multiple feasible re-allocations it will find the re-allocation with the minimum number of re-assignments.
\end{prop}

\begin{proof} 
This result is due to completeness and optimality properties of the BFS algorithm. The graph $\ccalG$ captures all possible re-assignments by construction. The proposed algorithm relies on BFS to find a path $p$ in this graph, modeling a sequence of task re-assignments that respects the constraints discussed in Section \ref{sec:taskRealloc2}. Since BFS is complete and optimal, so is Algorithm \ref{alg:bfs} with respect to fixing a Boolean formula $b_{q_B',q_B''}^d$.
\end{proof}




\begin{prop}[Completeness/Optimality of Alg. \ref{alg:RP}]\label{prop:ComplOpt}
Consider a failed predicate $\pi\in\mathcal{AP}_F$ and a set of NBA edges $\ccalE$ whose activation depends on $\pi$. If there exists feasible re-allocations (i.e., paths $p$) of all predicates for all Boolean formulas $b_{q_B',q_B''}^d$ for a given edge $e\in\ccalE$ and for all $e\in\ccalE$, then Alg. \ref{alg:RP} will find them. It will also find the solution with the minimum number of re-assignments across all sub-formulas and edges.
%
\end{prop}
\begin{proof}
Notice that given a failed predicate, all edges $e$ are repaired independently by Alg. \ref{alg:RP}. The reason is that the constraints imposed by the function $V_{q_B',q_B''}^d$ in Section \ref{sec:taskRealloc2} on repairing the failed predicates in $b_{q_B',q_B''}^d$ are fully independent across the sub-formulas $b_{q_B',q_B''}^d$ for a given edge and across edges as well. Then, the result holds due to Proposition \ref{bfsCompl}.
\end{proof}

\begin{prop}[Soundness/Completeness of Replanning]
Consider a re-allocation of predicates to robots by Alg. \ref{alg:RP}. If the re-planning algorithm (Section \ref{sec:onlineReplan}) returns a new prefix-suffix plan, then this plan satisfies $\phi$. Also, if there exists a feasible prefix-suffix plan satisfying the revised formula, then the re-planning algorithm will find it.
\end{prop}

\begin{proof} 
The online designed plan $\tau_H$ satisfies $\phi$ since, by construction, its suffix part includes a state $\bbq=(\bbp,\bbs,q_B)$ where $q_B$ is a final NBA state. As a result, the word generated along the plan satisfies the NBA accepting condition. 
Also, if there exists a feasible plan $\tau_H$, the proposed algorithm will find it. The reason is that in the worst-case scenario, where the current plan cannot be locally revised, a new prefix suffix plan is constructed from scratch by building a new tree using the sampling-based planner \cite{luo2021abstraction}; see Section \ref{sec:onlineReplan}. The employed sampling-based planner is (probabilistically) complete and, therefore, in this worst-case scenario, if there exists a feasible prefix-suffix plan $\tau_H$, then \cite{luo2021abstraction} will find it, completing the proof.
\end{proof}

\begin{rem}[Limitations]\label{rem:limitations}
%
%
The constraints defined in Section \ref{sec:taskRealloc2} cannot ensure the feasibility of the revised LTL formula. In fact, satisfying these constraints is necessary, but not sufficient, to ensure the existence of robot paths that can result in accepting NBA runs. We note that this is also a common limitation in related task allocation algorithms.
Additional  constraints can be introduced in the task re-allocation process, modeling necessary conditions to ensure that a transition is physically realizable; see e.g., Def. 3.4 in \cite{kantaros2020reactive}. 
%
\end{rem}

%% file: files/sim_new.tex
In this section, we provide experiments to demonstrate the performance of the proposed algorithm in the presence of unexpected multiple failures.
 We conducted our experiments in python 3 on a computer with Intel Core i7-8565U 1.80GHz and 16Gb RAM.  \textcolor{black}{Videos can be found in \cite{SimResilient}}. 



\subsection{Pipeline Inspection Task - Single Failure}\label{sec:sim3robot}
We revisit the pipeline inspection task discussed in Ex. \ref{ex:setup}. The considered LTL mission corresponds to an NBA with $8$ states.
\textcolor{black}{Due to failure of the camera for robot $3$, $\pi_2$, that was originally assigned to robot $3$, cannot be satisfied. At this point, Algorithm \ref{alg:RP} is called to fix all NBA edges associated with $\pi_2$. The total number of affected edges is $6$. In all edges, $\pi_2$ was assigned to robot $2$. 
The time needed for this reassignment is $0.0009$ secs and the new paths are generated in $0.26$ secs.} Observe in Fig. \ref{fig:3rd} robot $2$ completed its original task of operating the valve at $\ell_2$ and then took over the role of robot $3$ by visiting $\ell_3$ to take a photo. \textcolor{black}{The results for the case study discussed in Ex. \ref{ex:LTL} can be found in \cite{SimResilient}.}

\subsection{Aerial sensing task - Multiple failures}\label{sec:aerialGazebo}
We consider a mission involving $N=7$ drones with dynamics as in \cite{Furrer2016}. 
The abilities of the drones are defined as $c_1$, $c_2$, $c_3$, and $c_4$ pertaining to mobility, data transmission capability, photo-taking skills using a camera, and infrared imaging, respectively.
%
Drones $1$, $2$ and $3$ have abilities $c_1$ and $c_2$. Drones $4$ and $5$ have abilities $c_1$, $c_2$, and $c_3$, and drones $6$ and $7$ have abilities $c_1$, $c_3$, $c_4$. 
The drones are responsible for accomplishing an aerial sensing task while transmitting the collected data to a base station. 

Specifically, each of drones $1$, $2$, and $3$ should  \textcolor{black}{eventually always be present at locations $\ell_1$, $\ell_2$, and $\ell_3$, respectively,} to be able to transmit data collected from other drones to a base station (being close to $\ell_1$) in a multi-hop fashion.
Each of the remaining drones needs to complete two tasks one after the other. Drones $4$ and $5$ have to take photos at two locations each, while drones $5$ and $6$ need to take infrared locations at two locations each. This mission is captured by the following formula: $ \phi =\Diamond(\square\xi_1\wedge\pi_1\wedge\Diamond(\pi_2))\wedge\Diamond(\square\xi_1\wedge\pi_3\wedge\Diamond(\pi_4))\wedge\Diamond(\square\xi_1\wedge\pi_5\wedge\Diamond(\pi_6))\wedge\Diamond(\square\xi_1\wedge\pi_7\wedge\Diamond(\pi_8))$,
where $\xi_1$ is a Boolean formula \textcolor{black}{modeling the transmission objective} defined as 
$\xi_1 =  \pi_{\ccalT_2}(1, c_2, \ell_1)\wedge\pi_{\ccalT_2}(2, c_2, \ell_2)\wedge\pi_{\ccalT_2}(3, c_2, \ell_3)$, and $\pi_1=\pi_{\ccalT_3}(4, c_3, \ell_4)$,
$\pi_2=\pi_{\ccalT_3}(4, c_3, \ell_5)$,
$\pi_3=\pi_{\ccalT_3}(5, c_3, \ell_6)$,
$\pi_4=\pi_{\ccalT_3}(5, c_3, \ell_7)$,
$\pi_5=\pi_{\ccalT_4}(6, c_4, \ell_8)$,
$\pi_6=\pi_{\ccalT_4}(6, c_4, \ell_9)$,
$\pi_7=\pi_{\ccalT_4}(7, c_4, \ell_{10})$,
$\pi_8=\pi_{\ccalT_4}(7, c_4, \ell_{11})$.
This mission corresponds to an NBA with $132$ states.

We simulate the failure of all skills of two drones $2$ and $4$ at $t=9$ (when the robots are currently in the initial NBA state), requiring reassignment of their associated atomic predicates $\phi$. Due to this, $|\ccalE|=892$ NBA edges need to be repaired by Algorithm \ref{alg:RP}.
%
Algorithm \ref{alg:RP} reassigns drone $2$'s task of transmission to drone $5$ and drone $4$ and $5$'s original photographing tasks to drones $6$ and $7$, respectively. This reassignment process took $2.28$ secs 
and planning the new paths took $1.1$ secs. Toward the end of the mission, at $t=58$, robot $7$ failed completely as well. In this case, Algorithm \ref{alg:RP} had to fix one failed predicate $\pi_2$ (which was originally assigned to robot 4 and later taken over by robot $7$) 
and one edge. 
Alg. \ref{alg:RP} reassigns this predicate to robot $6$ in $0.0013$ secs. We note here that since the total number of NBA states is large ($132$), a failure in the initial part of the mission (when the robots are in the initial NBA state) leads to a large set $\hat{\ccalQ}_B^{\text{cur}}$ (see  Section \ref{sec:taskRealloc0}) and, consequently to large 
set $\ccalE$ of NBA edges that need to be repaired. As a result, the runtime required to fix the failed predicates in that case requires more time compared to failures toward the end of the mission where only a few edges needed to be fixed. A screenshot of the simulation is shown in Fig. \ref{fig:Gazebo_fig}.

\begin{figure}[t]
  \centering
\includegraphics[width=1\linewidth]{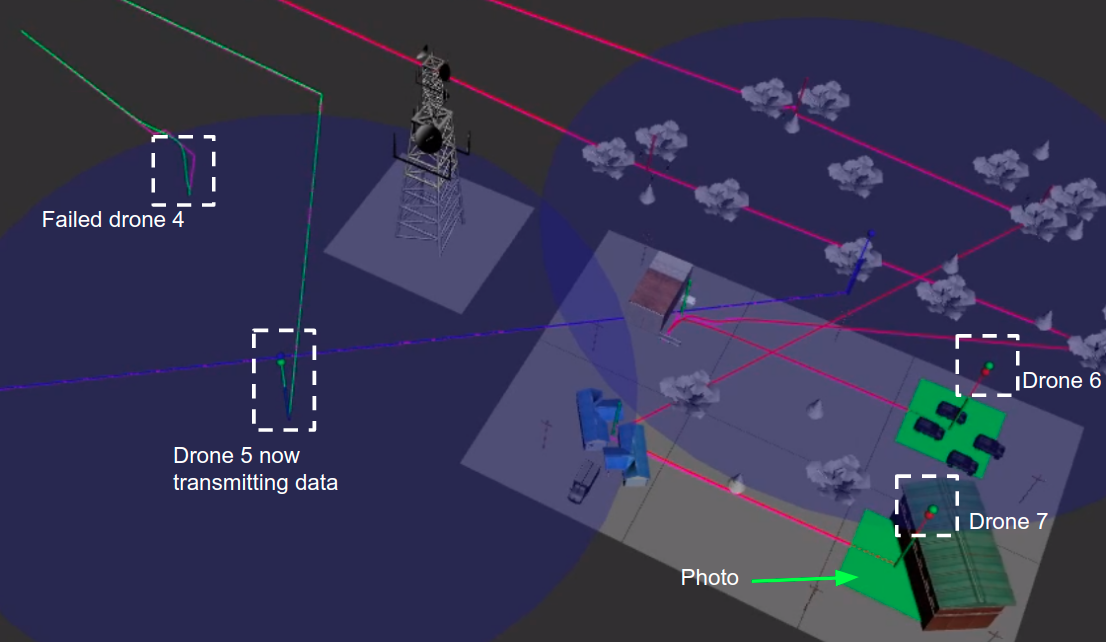}
\caption{\textcolor{black}{Drone $5$ performs a transmission task (blue circles denote transmission range) originally assigned to the failed drone $4$. Drones $6$ \& $7$ take photos completing tasks initially assigned to the failed drones $4$ \& $5$. 
}}
\label{fig:Gazebo_fig}
\end{figure}

\subsection{Failures in Large Robot Teams}
Finally, we consider a team of $N=24$ ground robots with $|\ccalC|=6$ skills. These skills are related to mobility, shutting off valves, extinguishing fires, collecting samples, and taking thermal and normal images.
The robots reside in a factory after a disaster and  they need to accomplish a sequence of various tasks to bring the situation under control as captured in: $\phi=\Diamond(\xi_{1}\wedge\Diamond(\xi_{2}\wedge\Diamond\xi_{3}))\wedge\Diamond(\xi_{4}\wedge\Diamond(\xi_{5}\wedge\Diamond\xi_{6}))$,
where each $\xi_i$ is a Boolean formula defined over atomic predicates of the form \eqref{eq:pip} for various robots as in Section \ref{sec:aerialGazebo}.
The considered formula corresponds to an NBA with $25$ states. At time $t=20$, all skills of $8$ robots completely failed requiring to fix $|\ccalE|=19$ NBA edges. The re-allocation and the re-planning process took $0.004$ secs and $17$ secs, respectively. 